\documentclass[conference]{IEEEtran}
\IEEEoverridecommandlockouts
\usepackage{cite}
\usepackage{amsmath,amssymb,amsfonts}
\usepackage{graphicx}
\usepackage{textcomp}
\usepackage{xcolor}
\usepackage{outlines}
\usepackage{stfloats}
\usepackage[inline]{enumitem}

\def\BibTeX{{\rm B\kern-.05em{\sc i\kern-.025em b}\kern-.08em
    T\kern-.1667em\lower.7ex\hbox{E}\kern-.125emX}}

\usepackage{mathtools,xparse}
\DeclarePairedDelimiter{\abs}{\lvert}{\rvert}

\usepackage{xcolor}
\usepackage[latin9]{inputenc}
\usepackage{xcolor}
\usepackage{array}
\usepackage{verbatim}
\usepackage{algorithm}
\usepackage[noend]{algpseudocode}
\algblock{ParFor}{EndParFor}
\algnewcommand\algorithmicparfor{\textbf{for}}
\algnewcommand\algorithmicpardo{\textbf{pardo}}
\algnewcommand\algorithmicendparfor{}
\algrenewtext{ParFor}[1]{\algorithmicparfor\ #1\ \algorithmicpardo}
\algrenewtext{EndParFor}{\algorithmicendparfor}
\makeatletter
\ifthenelse{\equal{\ALG@noend}{t}}%
  {\algtext*{EndParFor}}
  {}%
\makeatother

\algnewcommand{\LineComment}[1]{\State \(\textcolor{blue}{\triangleright}\) #1}
\usepackage{amsmath}
\usepackage{amsthm}
\usepackage{booktabs}
\usepackage{amssymb}
\usepackage{graphicx}
\usepackage{outlines}
\usepackage{mathtools}
\usepackage{caption}
\usepackage{subcaption}
\usepackage{braket}

\usepackage{authblk}
\usepackage{bm}
\PassOptionsToPackage{normalem}{ulem}
\usepackage{ulem}
\makeatletter
\def\Let@{\def\\{\notag\math@cr}}
\makeatother
\usepackage{mathtools}

\newtheorem{theorem}{Theorem}

\newcommand{\V}{\mathcal{V}}
\newcommand{\Vl}{\mathcal{V}^{(\ell)}}
\newcommand{\G}{\mathcal{G}}
\newcommand{\Gsub}{\mathcal{G}_{\text{sub}}}
\newcommand{\E}{\mathcal{E}}
\newcommand{\Vls}{\mathcal{V}_{\text{LS}}}

\newcommand{\Vgs}{\mathcal{V}_{\text{GS}}}

\newcommand{\Als}{\bm{A}_{\text{LS}}}

\newcommand{\Ags}{\bm{A}_{\text{GS}}}
\newcommand{\Hls}{\bm{H}_{\text{LS}}}

\newcommand{\Hgs}{\bm{H}_{\text{GS}}}

\newcommand{\Wself}{\bm{W}_{\text{self}}}
\newcommand{\Wneigh}{\bm{W}_{\text{neigh}}}

\newcommand{\Hh}{\bm{H}}
\newcommand{\Els}{\mathcal{E}_{\text{LS}}}

\newcommand{\Label}{\bm{L}}

\newcommand{\dls}{d_{\text{LS}}}

\newcommand{\fl}{f^{(\ell)}}
\newcommand{\fll}{f^{(\ell+1)}}

\newcommand{\Vsub}{\V_{\text{sub}}}
\newcommand{\Esub}{\E_{\text{sub}}}

\newcommand{\paren}[1]{\left( #1 \right)}

\newcommand{\size}[1]{\left\lvert #1 \right\rvert}

\begin{document}


\title{Accurate, Efficient and Scalable Graph Embedding\\
}
\author[1]{Hanqing Zeng\IEEEauthorrefmark{1}\thanks{\IEEEauthorrefmark{1}Equal contribution}}
\author[1]{Hongkuan Zhou\IEEEauthorrefmark{1}}
\author[1]{Ajitesh Srivastava}
\author[2]{Rajgopal Kannan}
\author[1]{Viktor Prasanna}
\affil[1]{Univerisity of Southern California\\
\{zengh,hongkuaz,ajiteshs,prasanna\}@usc.edu}
\affil[2]{US Army Research Lab-West\\
rajgopal.kannan.civ@mail.mil}
\maketitle

\begin{abstract}

The Graph Convolutional Network (GCN) model and its variants are powerful graph embedding tools for facilitating classification and clustering on graphs. However, a major challenge is to reduce the complexity of layered GCNs and make them parallelizable and scalable on very large graphs --- state-of the art techniques are unable to achieve scalability without losing accuracy and efficiency. In this paper, we propose novel parallelization techniques for graph sampling-based GCNs that achieve superior scalable performance on very large graphs without compromising accuracy. Specifically, our GCN guarantees work-efficient training and produces order of magnitude savings in computation and communication. To scale GCN training on tightly-coupled shared memory systems, we develop parallelization strategies for the key steps in training:  For the graph sampling step, we exploit parallelism within and across multiple sampling instances, and devise an efficient data structure for concurrent accesses that provides theoretical guarantee of near-linear speedup with number of processing units. For the feature propagation step within the sampled graph, we improve cache utilization and reduce DRAM communication by data partitioning. We prove that our partitioning strategy is a 2-approximation for minimizing the communication time compared to the optimal strategy. We demonstrate that our parallel graph embedding outperforms state-of-the-art methods in scalability (with respect to number of processors, graph size and GCN model size), efficiency and accuracy on several large datasets. On a 40-core Xeon platform, our parallel training achieves $64\times$ speedup (with AVX) in the sampling step and $25\times$ speedup in the feature propagation step, compared to the serial implementation, resulting in a net speedup of $21\times$. Our scalable algorithm enables deeper GCN, as demonstrated by $1306\times$ speedup on a 3-layer GCN compared to Tensorflow implementation of state-of-the-art.
\end{abstract}

\setlength{\textfloatsep}{10pt plus 1.0pt minus 2.0pt}

\section{Introduction}


Graph embedding is a powerful dimensionality reduction technique. Taking an unstructured, attributed graph as input, the embedding process outputs structured vectors which capture information of the original graph. 
Embedding facilitates data mining on graphs, and thus is essential in a wide range of tasks such as content recommendation and protein function prediction. Among the various embedding techniques, Graph Convolutional Network (GCN) \cite{gcn} and its variants \cite{graphsage}, \cite{fastgcn} have attained much attention recently. GCN based approaches produce accurate and robust results without the need of manual feature selection.


In order to scale GCN to large graphs, the typical approach is to decompose training into ``mini-batches" and attempt to parallelize mini-batch training by sampling on GCN layers (layer sampling). The batched GCN\cite{gcn} and its successor GraphSAGE\cite{graphsage} sample the inter-layer edge connections. Their approaches preserve the training accuracy of the original GCN model, but their parallelization strategy is not work-efficient. The amount of redundant computation increases by a factor of node degree for every additional layer of GCN. To alleviate such high redundancy due to ``neighbor explosion'' in deeper layers, FastGCN \cite{fastgcn} proposes to sample the nodes of GCN layers instead of the edge connections. Although this approach is empirically faster than \cite{gcn}, \cite{graphsage}, it does not guarantee work-efficiency, incurs accuracy loss and requires expensive preprocessing which affects scaling.

Due to the layer sampling design philosophy, it is difficult for state-of-the-art methods \cite{gcn,graphsage,fastgcn} to simultaneously achieve accuracy, efficiency and scalability. In this work, we propose a new graph sampling-based GCN which achieves superior performance on a variety of large graphs. Our novelty lies in designing scalable GCN model based on \emph{parallelized graph sampling} (rather than layer sampling), {\it without compromising} accuracy or efficiency. We achieve scalability and performance by 1) Developing a novel data structure that enables efficient parallel subgraph sampling through fast parallel updates to degree distributions; 2) Developing an optimized parallel implementation of intra-subgraph feature propagation and weight updates --- specifically a cache-efficient partitioning scheme that reduces DRAM communication, resulting in optimized memory performance. 
We achieve work-efficiency by avoiding neighbor explosion, as each layer of our complete GCN consists of only small set of nodes. 
Finally, since our sampled subgraphs preserve connectivity characteristics of the original training graph, we demonstrate that accuracy also remains unaffected.
The main contributions of this paper are:

\begin{outline}
\1 We propose parallel training algorithm for a novel graph sampling-based GCN model, where:
    \2 \emph{Accuracy} is achieved due to connectivity-preserving graph sampling techniques.
    \2 \emph{Efficiency} is optimal since ``neighbor explosion" in layer sampling based GCNs is avoided.
    \2 \emph{Scalability} is achieved with respect to number of processing cores, graph size and GCN depth by extracting parallelism at various key steps.
\1 We propose a novel data structure that accelerates degree distribution based sampling through fast incremental parallel updates. Our method has a near-linear scalability guarantee with respect to number of processing units. 
\1 We parallelize the key training step of subgraph feature propagation. Using intelligent partitioning on the feature dimension, we optimize DRAM and cache performance and scale training to large number of cores.
\1 We perform thorough evaluation on a 40-core Xeon server. Compared with serial implementation, we achieve $21\times$ overall training time speedup. Compared with other GCN based methods, our parallel (serial) implementation achieves up to $37.4\times$ ($7.8\times$) training time reduction when reaching the same accuracy.
\1 We demonstrate that our implementation can enable deeper GCN. We obtain $1306\times$ speedup for 3-layer GCN compared to Tensorflow implementation of state-of-the-art on a 40-core machine.
\end{outline}
\section{Background and Related Work}





\label{sec: gcn background}

Graph Convolutional Network (GCN) along with its variations is the state-of-the-art embedding method, widely shown to yield highly accurate and robust results. A GCN is a kind of multi-layer neural network, and it performs vertex embedding as follows. 
An input graph with vertex attributes (or features) is fed into the GCN. The GCN propagates the features layer by layer, where each layer performs feature extraction based on the learned weight matrices and the input graph topology. The last GCN layer outputs embedding vectors for the vertices of the input graph. 
Essentially, both the input graph attributes and topology are ``embedded" into the output vectors. 

\begin{figure*}
\centering
\includegraphics[width=0.93\textwidth]{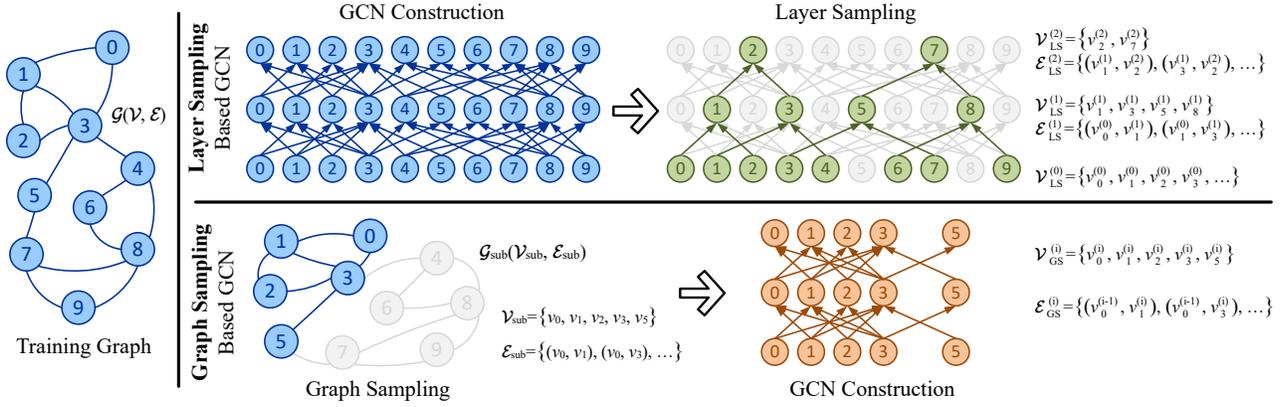}
\caption{Illustration on layer sampling and graph sampling based GCN design.}
\label{fig: illustration}
\end{figure*}

Figure \ref{fig: illustration} shows the network structure of GCNs. We use superscript ``$(\ell)$" to denote GCN layer-$\ell$ parameters. An $L$-layer GCN upon the input graph $\G(\V,\E)$ is constructed as follows. Each layer consists of $\abs{\V}$ nodes ---  layer node $v_i^{(\ell)}\in \Vl$ corresponds to the graph vertex $v_i\in \V$. 
Edges in GCN connect nodes in adjacent layers based on the input graph edges. Specifically, the collection of edges between layers $\ell-1$ and $\ell$ are defined by $\E^{(\ell)}\in\Set{\paren{v_i^{(\ell-1)},v_j^{(\ell)}}|\paren{v_i,v_j}\in\E}$. 
Note that $\V^{(\ell-1)}$, $\V^{(\ell)}$ and $\E^{(\ell)}$ form a bipartite graph represented by bi-adjacency matrix $\bm{A}^{(\ell)}$. 
Each GCN node $v_i^{(\ell)}$ is associated with a feature vector $\bm{h}_i^{(\ell)}\in\mathbb{R}^{\fl}$ of length $\fl$, and let $\Hh^{(\ell)}\in\mathbb{R}^{\abs{\V}\times \fl}$ be the feature matrix. 
Note that $\Hh^{(0)}$ consists of vertex attributes of input graph $\G$.

The propagation of $\bm{h}_i^{(\ell)}$ from layer $\ell-1$ to $\ell$ is processed by two weight matrices: $\Wself^{(\ell)}$ and $\Wneigh^{(\ell)}$. Both weight matrices are learned during training. 
The forward propagation of GCN from layer $\ell-1$ to $\ell$ involves two major steps:

\begin{enumerate}
    \item \emph{Feature aggregation}: Each $v_i^{(\ell)}$  collects features of its layer $\ell-1$ neighbors, $\Set{\bm{h}_j^{(\ell-1)}|\paren{v_j^{(\ell-1)},v_i^{(\ell)}}\in\E^{(\ell)}}$, and calculates the mean of neighbor features.
    \item \emph{Weight application}: The mean-aggregated neighbor features are multiplied by $\Wneigh^{(\ell)}$. The self-features $\bm{h}_i^{(\ell-1)}$ from layer $\ell-1$ are also multiplied by the $\Wself^{(\ell)}$. 
\end{enumerate}

The features after applying weights are sent to layer $\ell$. With an optional neighbor-self feature concatenation and non-linear activation, we obtain the next layer features $\Hh^{(\ell)}$. 

To train GCNs on large scale graphs, it is essential to do mini-batch training, where number of vertices involved in each weight update is much smaller than the training graph size. Batched GCN \cite{gcn}, GraphSAGE \cite{graphsage} and FastGCN \cite{fastgcn} incorporate various GCN \emph{layer sampling} techniques to construct mini-batches. Upper part of Figure \ref{fig: illustration} abstracts the edge based layer sampling technique \cite{graphsage} and node based layer sampling technique \cite{fastgcn}, where subscript LS dentes \textbf{L}ayer \textbf{S}ampling. 
Since layer sampling requires each node in layer $\ell$ to select multiple neighbor nodes in layer $\ell-1$, the deeper the layer sampler goes, the more nodes will be sampled (i.e., $\size{\Vls^{(\ell)}}$ becomes larger when $\ell$ gets smaller). We refer to this phenomenon as ``neighbor explosion''. 
On the other hand, \cite{fastgcn} proposes a node based layer sampler. It samples in a two phased fashion. The first phase samples nodes of all the $L$ layers based on a pre-calculated probability distribution, where the samplers of each layer are independent. The second phase re-constructs the inter-layer edges to connect the sampled nodes. Empirically, this method mitigates the ``neighbor explosion'' problem  at the cost of accuracy loss and potentially expensive pre-processing (calculating probability distribution). Deeper layers may still require larger sampling population to avoid overly sparse inter-layer connection.

\section{Graph Sampling-Based GCN Model}

We present a novel graph sampling-based GCN model. Our parallel mini-batch training outperforms state-of-the-art GCN in accuracy, efficiency and scalability, simultaneously. We present the design of the graph sampling-based GCN (Section \ref{sec: gcn design}), and analyze the advantages in efficiency (Section \ref{sec: gcn efficiency}) and accuracy (Section \ref{sec: gcn accuracy}). We then present optimizations to scale training on parallel machines (Sections \ref{sec: para sample}, \ref{sec: para spmm}).

\subsection{Design of the Model}
\label{sec: gcn design}

As illustrated by the lower part of Figure \ref{fig: illustration}, the graph sampling-based approach does not construct a GCN directly on the original input graph $\G$. Instead, for each iteration of weight update during training, we first sample a small induced subgraph $\Gsub(\Vsub,\Esub)$ from $\G(\V,\E)$. We then construct a \emph{complete} GCN on $\Gsub$. The forward and backward propagation are both on this small yet complete GCN. Algorithm \ref{algo: gsaint training} presents the details of our approach, where subscript GS denotes \textbf{G}raph \textbf{S}ampling. The key distinction from traditional GCNs is that the computations (lines 5-13) are performed on nodes of the sampled graph instead of the sampled layer nodes, thus requiring much less computation in training (Section \ref{sec: gcn efficiency}). We also discuss the requirement for the $\text{SAMPLE}_{\text{G}}$ function (line 3) in Section \ref{sec: gcn efficiency}, and present a representative sampling algorithm that leads to high accuracy.

\begin{algorithm}
\caption{Training algorithm for mini-batched GCN using graph sampling techniques}
\label{algo: gsaint training}
\begin{algorithmic}[1]
\renewcommand{\algorithmicrequire}{\textbf{Input:}}
\renewcommand{\algorithmicensure}{\textbf{Output:}}
\Require Training graph $\G(\V,\E,\Hh^{(0)})$; Labels $\Label$
\Ensure Trained weights $\Set{\Wself^{(\ell)},\Wneigh^{(\ell)}|1\leq\ell\leq L}$
\While{not terminate}\color{blue}\Comment{Iterate over mini-batches}\color{black}
    \State $\G_{\text{sub}}(\V_\text{sub},\E_\text{sub})\gets \text{SAMPLE}_{\text{G}}\big(\G(\V,\E)\big)$
    \State $\Set{\Vgs^{(\ell)}}, \Set{\E_{\text{GS}}^{(\ell)}}\gets $ GCN construction on $\G_{\text{sub}}$
    \State $\Set{\Ags^{(\ell)}}\gets $ bi-adj matrix of $\Set{\Big(\Vgs^{(\ell-1)},\Vgs^{(\ell)}, \E_{\text{GS}}^{(\ell)}\Big)}$
    \State $\Hgs^{(0)}\gets \Hh^{(0)}\left[ \Vgs^{(0)}\right]$
    \For{$\ell=1$ to $L$}\color{blue}\Comment{Forward propagation}\color{black}
        \State $\Hh_{\text{neigh}}\gets \Big(\Ags^{(\ell)}\Big)^{T}\cdot\Hgs^{(\ell-1)}\cdot\Wneigh^{(\ell)}$
        \State $\Hh_{\text{self}}\gets \Hgs^{(\ell-1)}\cdot \Wself^{(\ell)}$
        \State $\Hgs^{(\ell)}\gets \sigma\Big(\Hh_{\text{neigh}}|\Hh_{\text{self}}\Big)$\color{blue}\Comment{Concat \& activation}\color{black}
    \EndFor
    \LineComment{\color{blue}Label prediction by embedding; SGD weight update\color{black}}
    \State $\bm{X}\gets \text{PREDICT}\big(\Hgs^{(L)}\big)$
    \State $C\gets \text{LOSS}\paren{\bm{X},\bm{L}\left[ \Vgs^{(L)} \right]}$
    \State $\text{ADAM}\Big(C,\bm{X},\Set{\Hgs^{(\ell)}},\Set{\Wneigh^{(\ell)}}, \Set{\Wself^{(\ell)}}\Big)$
\EndWhile
\State \Return $\Set{\Wself^{(\ell)}},\Set{\Wneigh^{(\ell)}}$
\end{algorithmic} 
\end{algorithm}

\subsection{Complexity of Graph Sampling-Based GCN}
\label{sec: gcn efficiency}

We analyze the computation complexity of our graph-sampling based GCN and show that it achieves work-optimality. In the  following analysis, we do not consider the sampling overhead, and we only focus on the forward propagation, since backward propagation has identical computation characteristics as forward propagation. Later, we also experimentally demonstrate that our technique is significantly faster even with the sampling step included (see Section~\ref{sec: exp}).

The main operations to propagate a batch (consisting of vertices of the sampled graph) forward by one layer are:
\begin{itemize}
    \item Feature aggregation: Each feature vector from layer-$\ell$ propagates via edges in $\E_{\text{GS}}^{(\ell)}$ to be aggregated in the next layer, leading to $\mathcal{O}\paren{\size{\E_{\text{GS}}^{(\ell)}}\cdot f^{(\ell)}}$ operations.
 \item Weight application: Each vertex multiplies its feature with the weight, leading to 
$\mathcal{O}\paren{\size{\V_{\text{GS}}^{(\ell)}}\cdot f^{(\ell-1)} f^{(\ell)}}$ operations.
\end{itemize}

Complexity of $L$-layer forward propagation in one batch is:

\begin{align}
\label{eq: Dcmp batch}
\mathcal{O}\left(\sum\limits_{\ell=0}^{L-1}\Big(\size{\E_{\text{GS}}^{(\ell)}}\cdot \fl+\size{\V^{(\ell+1)}_{\text{GS}}}\cdot\fl\cdot \fll\Big)\right)\\
\end{align}

The GCN constructed on our sampled subgraph has the same set of nodes in all layers $\size{\Vgs^{(\ell)}}=\size{\Vsub}$. If the average degree of the subgraph is $d_\text{GS}$, then number of edges between each layer of GCN $\E_\text{GS}^{(\ell)} = d_\text{GS}|\Vsub|$. This results in the complexity of a batch (subgraph) of our algorithm to $\mathcal{O}\paren{L\cdot \size{\V_\text{sub}}\cdot f\cdot (f+d_{\text{GS}})}$. If we define an epoch as one full traversal of all training vertices (i.e., processing of $|\V|/\size{\V_\text{sub}}$ batches), complexity of an epoch is $\mathcal{O}\paren{L\cdot \size{\V}\cdot f\cdot (f+d_{\text{GS}})}$

\paragraph*{Comparison Against Other GCN Models}
In \cite{graphsage}, a certain number $d_\text{LS}$ of neighbors are selected from the next layer  for each vertex in the current layer. It can be shown that depending on the batch size the complexity falls between: 

\textit{Case 1 [Small batch size]:} $\mathcal{O}\paren{\dls^{L}
\cdot \size{\V}\cdot f\cdot(f+\dls)}$.

\textit{Case 2 [Large batch size]}
$\mathcal{O}\paren{L\cdot \size{\V}\cdot f\cdot(f+\dls)}$.

We observe that when the batch size is much smaller than the training graph size, layer sampling technique by \cite{graphsage} results in high training complexity (it grows exponentially with GCN depth, although, often small values of L are used). This is the ``neighbor explosion'' phenomenon, which is effectively due to 
redundant computations for small batches, making the mini-batch training of \cite{graphsage} very inefficient. On the other hand, when the batch size of \cite{graphsage} becomes comparable to the training graph size, the training complexity grows linearly with GCN depth and training graph size. However, this resolution of the ``neighbor explosion'' problem comes at the cost of slow convergence and low accuracy \cite{large_batch_size}. Thus, such training configuration is not applicable to large scale graphs.

Our graph-sampling based GCN technique leads to a parallel training algorithm whose complexity is linear in GCN depth and training graph size. The work-efficiency of our GCN model is guaranteed by design, since the GCN in this case is always complete. In addition, by choosing proper graph sampling algorithms, we can construct subgraphs whose sizes are small, and do not need to be grown with the training graph (as shown in Section \ref{sec: exp}). Thus, the graph sampling-based GCN model achieves work-optimality without any sacrifice in accuracy due to large batch sizes.

\subsection{Accuracy of Graph Sampling-Based GCN}
\label{sec: gcn accuracy}

By design, layer-based sampling methods assume that a subset of neighbors of a given node is sufficient to learn its representation. We achieve the same goal by sampling the graph itself. If the sampling algorithm constructs enough number of representative subgraphs $\Gsub$, our graph sampling-based GCN model should absorb all the information in $\G$, and generate accurate embeddings.
More specifically,
as discussed in Section \ref{sec: gcn background}, the output vectors ``embed" the input graph topology as well as the vertex attributes. A good graph sampling algorithm, thus, should guarantee:

\begin{enumerate}
    \item Sampled subgraphs preserve the connectivity characteristics in the training graph. 
    \item Each vertex in the training graph has some non-negligible probability to be sampled;
\end{enumerate}


For the first requirement, while ``connectivity'' may have several definitions, subgraphs output by the well-known frontier sampling algorithm \cite{frontier} approximate the original graph with respect to multiple connectivity measures. In addition, the frontier sampling algorithm also satisfies the second requirement. 
At the beginning of the sampling phase, the sampler picks some initial root vertices uniformly at random from the original graph. These root vertices constitute a significant portion of the sampled graph vertices. Thus, over large enough number of sampling iterations, all input vertex attributes of the training vertices will be covered by the sampler. 

For the above reasons, we use frontier sampling algorithm as the sampler for our GCN model.  We also perform detailed accuracy evaluation on our GCN model using the frontier sampling algorithm. The results in Section \ref{sec: exp} prove empirically that our design does not incur any accuracy loss.



\section{Parallel Graph Sampling Algorithm}
\label{sec: para sample}

\subsection{Graph Sampling Algorithm}
\label{sec: para sample baseline}

The frontier sampling algorithm proceeds as follows. 
Throughout the sampling process, the sampler maintains a constant size frontier set $\text{FS}$ consisting of $m$ vertices in $\G$. In each iteration, the sampler randomly pops out a vertex $v$ in $\text{FS}$ using a degree based probability distribution, and replaces $v$ in $\text{FS}$ with a randomly selected neighbor of $v$. The popped out $v$ is added to the vertices $\V_{\text{sub}}$ of $\Gsub$. 
The sampler repeats the above update process on the frontier set $\text{FS}$, until the size of $\V_{\text{sub}}$ reaches the desired budget $n$. Algorithm \ref{algo: frontier sampling} shows the details. According to \cite{frontier}, a good empirical value of $m$ is $1000$.

\begin{algorithm}
\caption{Frontier sampling algorithm}
\label{algo: frontier sampling}
\begin{algorithmic}[1]
\renewcommand{\algorithmicrequire}{\textbf{Input:}}
\renewcommand{\algorithmicensure}{\textbf{Output:}}
\Require{Original graph $\G(\V,\E)$; Frontier size $m$; Budget for number of sampled vertices $n$}
\Ensure{Induced subgraph $\Gsub(\V_{\text{sub}},\E_{\text{sub}})$}
\State $\text{FS}\gets m$ vertices selected uniformly at random from $\V$
\State $\V_{\text{sub}}\gets \Set{v_i|v_i\in \text{FS}}$
\For{$i=0$ to $n-m-1$}
    \State Select $u\in \text{FS}$ with probability $\text{deg}(u)/\sum_{v\in \text{FS}}\text{deg}(v)$
    \State Select $u'$ from $\Set{w|(u,w)\in\E}$ uniformly at random
    \State $\text{FS}\gets \big(\text{FS}\setminus \Set{u}\big)\cup\Set{u'}$
    \State $\V_{\text{sub}}\gets \V_{\text{sub}}\cup\Set{u}$
\EndFor
\State $\Gsub\gets $ Subgraph of $\G$ induced by $\V_{\text{sub}}$ 
\State \Return $\Gsub(\V_{\text{sub}},\E_{\text{sub}})$ 
\end{algorithmic} 
\end{algorithm}

In our experiments sequential implementation, we notice that about half of the time  is spent in sampling phase. Therefore, there is a need to parallelize the sampling step. The challenges to parallelization are: 
\begin{enumerate*}
    \item While sampling from a discrete distribution is a well-researched problem, we are focused on fast parallel sampling of a {\it dynamic} degree distribution (changes with addition/deletion of new vertex to the frontier). Existing well-known methods for fast sampling such as aliasing (which can output a sample in $\mathcal{O}(1)$ time with linear processing) cannot be modified easily for this problem.
    It is non-trivial to select a vertex from evolving $\text{FS}$ with low complexity. 
    A straightforward implementation requires $\mathcal{O}(m\cdot n)$ work to sample a single $\Gsub$, 
    which is expensive given $m=1000$. 
    \item The sampler is inherently sequential. The vertices in the frontier set should be popped out one at a time to guarantee the quality of the sampled graph $\Gsub$.
\end{enumerate*}

To address the above challenges, we first propose a novel data structure that lowers the complexity of frontier sampler and allows thread-safe parallelization (Section \ref{sec: para sample intra}). We then propose a training scheduler that exploits parallelization within and across sampler instances (Section \ref{sec: para sample inter}). 

\subsection{Dashboard Based Implementation}
\label{sec: para sample intra}

Since vertices in the frontier set is replaced only one at a time, an efficient implementation should allow incremental update of the probability distribution. We propose a ``Dashboard" table to store the status of current and historical frontier vertices (a vertex becomes a historical frontier vertex after it gets popped out of the frontier set). The vertex to be popped out next is selected by probing the Dashboard using randomly generated indices. The data structure ensures that the probability distribution for picking vertices is updated quickly and incrementally.
Below we formally describe the data structure and operations in the Dashboard-based sampler. The implementation involves two arrays: 

\begin{itemize}
    \item Dashboard $\text{DB}\in\mathbb{R}^{3\times (\eta\cdot m\cdot d)}$: A table maintaining the status and probabilities of current and historical frontier vertices. An entry corresponding to $v_i$ has 3 slots. The first slot is $i$; the second slot is an offset value, helping with the invalidation of all of the $v_i$ related entries when $v_i$ is popped out; the third slot is a value $k$, if $v_i$ is the $k^{th}$ vertex added into DB. Parameter $\eta$ is explained later. 

    \item Index array $\text{IA}\in \mathbb{R}^{2\times (\eta\cdot m\cdot d + 1)}$: An auxiliary array to help cleanup $\text{DB}$ when the Dashboard overflows. The $j^{th}$ entry in IA has 2 slots, the first slot records the starting index of the DB entries corresponding to $v$, where $v$ is the $j^{th}$ vertex added into DB. The second slot is a Boolean flag, which is \verb|True| when $v$ is a current frontier vertex, and \verb|False| when $v$ is a historical frontier vertex.
\end{itemize}

Since the probability of popping out a vertex in frontier is proportional to its degree, we allocate $\text{deg}(v_i)$ continuous entries in $\text{DB}$, for each $v_i$ currently in the frontier set. This way, the sampler only needs to probe DB uniformly at random to achieve line 4 of Algorithm \ref{algo: frontier sampling}. Thus, $\text{DB}$ should contain at least $m\cdot d$ entries, where $d$ is the average degree of vertices in frontier. For ease of incremental update of DB we append the entries for the new vertex and invalidate the entry of popped out vertex, instead of replacing one with the other. To accommodate this we introduce an enlargement factor $\eta$ ($\eta >1$), and set the length of $\text{DB}$ to be $\eta\cdot m\cdot d$ (as an approximation, we set $d$ as the average degree of the training graph $\G$).
As the sampling proceeds, eventually, all of the $\eta\cdot m\cdot d$ entries in DB may be filled up by current and historical frontier vertices. When this happens, we free up the space occupied by historical frontier vertices, and resume the sampler. Although ``cleanup" of the Dashboard is expensive, due to the factor $\eta$, such scenario does not happen frequently (see complexity analysis in Section~\ref{sec: para sample inter}).
Using the information in IA, the cleanup phase does not need to traverse all of the $\eta\cdot m\cdot d$ entries in DB to locate the space to be freed. When DB is full, the entries in DB can correspond to at most $\eta\cdot m\cdot d$ vertices. Thus, we safely set the capacity of IA to be $\eta\cdot m\cdot d$ + 1. Slot 1 of the last entry of IA contains the current number of used DB entries.


\subsection{Intra- and Inter-Subgraph Parallelization}
\label{sec: para sample inter}





Since our subgraph-based GCN approach requires independently sampling multiple  subgraphs, we can easily sample different subgraphs on different processors in parallel. Also, we can further parallelize within each instance of sampling by exploiting the parallelism in  probing, book-keeping  and cleanup of DB, which is presented next.

\begin{algorithm}[!ht]
\caption{Parallel Dashboard based frontier sampling}
\label{algo: frontier sampling par}
\begin{algorithmic}[1]
\renewcommand{\algorithmicrequire}{\textbf{Input:}}
\renewcommand{\algorithmicensure}{\textbf{Output:}}
\Require{Original graph $\G(\V,\E)$; Frontier size $m$; Budget $n$; Enlargement factor $\eta$; Number of processors $p$}
\Ensure{Induced subgraph $\Gsub(\V_{\text{sub}},\E_{\text{sub}})$}
\State $d\gets \abs{\E}/\abs{\V}$
\State $\text{DB}\gets$ Array of $\mathbb{R}^{3\times (\eta\cdot m\cdot d)}$ with value \verb|INV|\color{blue}\Comment{INValid}\color{black}
\State $\text{IA}\gets$ Array of $\mathbb{R}^{2\times (\eta\cdot m\cdot d + 1)}$ with value \verb|INV|
\State $\text{FS}\gets m$ vertices selected uniformly at random from $\V$
\State $\V_{\text{sub}}\gets \Set{v|v\in\text{FS}}$
\State $\text{FS}\gets $ Indexable list of vertices converted from set FS
\State $\text{IA}\lbrack 0,0 \rbrack\gets 0;\qquad\text{IA}\lbrack 1,0 \rbrack\gets $\verb|True|$;$
\For{$i=1$ to $m$}\color{blue}\Comment{Initialize IA from FS}\color{black}
    \State $\text{IA}\lbrack 0,i \rbrack\gets \text{IA}\lbrack 0,i-1 \rbrack+\text{deg}(\text{FS}\lbrack i-1 \rbrack)$
    \State $\text{IA}\lbrack 1,i \rbrack\gets (i\neq m)?\texttt{True}:\texttt{False}$
\EndFor

\ParFor{$i=0$ to $m-1$}\color{blue}\Comment{Initialize DB from FS}\color{black}
    \For{$k=\text{IA}\lbrack i \rbrack$ to $\text{IA}\lbrack i+1 \rbrack-1$}
        \State $\text{DB}\lbrack 0,k\rbrack\gets \text{FS}\lbrack i \rbrack$
        \State $\text{DB}\lbrack 1,k\rbrack\gets (k\neq\text{IA}\lbrack i \rbrack) ? (k-\text{IA}\lbrack i \rbrack):-\text{deg}(\text{FS}\lbrack i \rbrack)$
        \State $\text{DB}\lbrack 2,k\rbrack\gets i$
    \EndFor
\EndParFor

\State $s\gets m$
\For{$i=m$ to $n-1$}\color{blue}\Comment{Sampling main loop}\color{black}
    \State $v_{\text{pop}}\gets$ pardo\_POP\_FRONTIER$($DB,$p)$
    \State $v_{\text{new}}\gets $ Vertex sampled from $v_{\text{pop}}$'s neighbors
    \If{$\text{deg}(v_{\text{new}})>\eta\cdot m\cdot d-\text{IA}\lbrack 0,s \rbrack+1$}
        \State $\text{DB}\gets$ pardo\_CLEANUP$(\text{DB},\text{IA},p)$
        \State $s\gets m-1$
    \EndIf
    \State pardo\_ADD\_TO\_FRONTIER $(v_{\text{new}},s,\text{DB},\text{IA},p)$
    \State $\V_{\text{sub}}\gets \V_{\text{sub}}\cup\Set{v_{\text{new}}}$
    \State $s\gets s+1$
\EndFor

\State $\Gsub\gets $ Subgraph of $\G$ induced by $\V_{\text{sub}}$ 
\State \Return $\Gsub(\V_{\text{sub}},\E_{\text{sub}})$ 
\end{algorithmic} 
\end{algorithm}

\begin{algorithm}[!ht]
\caption{Functions in Dashboard Based Sampler}
\label{algo: frontier sampling func}
\begin{algorithmic}[1]
\Function{pardo\_POP\_FRONTIER}{$\text{DB},p$}
\State $\text{idx}_{\text{pop}}\gets $\verb|INV|\color{blue}\Comment{Shared variable}\color{black}
\ParFor{$j=0$ to $p-1$}
    \While{$\text{idx}_{\text{pop}}==$ \texttt{INV}}\color{blue}\Comment{Probing DB}\color{black}
        \State $\text{idx}_p\gets $ Index generated uniformly at random
        \If {$\text{DB}\lbrack 0,\text{idx}_p \rbrack\neq\texttt{INV}$}
            \State $\text{idx}_{\text{pop}}\gets \text{idx}_p$            
        \EndIf
    \EndWhile
    \State BARRIER
    \State $v_{\text{pop}}\gets \text{DB}\lbrack 0,\text{idx}_{\text{pop}} \rbrack$
    \State $\text{offset}\gets \text{DB}\lbrack 1,\text{idx}_{\text{pop}} \rbrack$
    \State $\text{idx}_{\text{start}}\gets (\text{offset}>0)?(\text{idx}_{\text{start}}-\text{offset}):\text{idx}_{\text{start}}$
    \State $\text{deg}\gets \text{DB}\lbrack 1,\text{idx}_{\text{start}} \rbrack$
    \For{$k=j\cdot \frac{\text{deg}}{p}$ to $(j+1)\cdot \frac{\text{deg}}{p}-1$}\color{blue}\Comment{Update DB}\color{black}
        \State $\text{DB}\lbrack 0,k+\text{idx}_{\text{start}} \rbrack\gets \texttt{INV}$
    \EndFor
    \State BARRIER
\EndParFor
\State $\text{IA}\lbrack 1,\text{DB}\lbrack 2, \text{idx}_{\text{pop}} \rbrack \rbrack\gets \texttt{False}$\color{blue}\Comment{Update IA}\color{black}
\State \Return $v_{\text{pop}}$
\EndFunction

\Function{pardo\_CLEANUP}{$\text{DB},\text{IA}$,$p$}
    \State $\text{DB}_\text{new}\gets $ New, empty dashboard
    \State $k\gets$ Cumulative sum of $\text{IA}\lbrack 0,: \rbrack$ masked by $\text{IA}\lbrack 1,: \rbrack$
    \ParFor {$i=0$ to $p-1$}
        \State Move entries from DB to $\text{DB}_\text{new}$ by offsets in $k$
    \EndParFor
    \ParFor {$i=0$ to $p-1$}
        \State Re-index IA based on $\text{DB}_\text{new}$    
    \EndParFor
    \Return $\text{DB}_\text{new}$
\EndFunction

\Function{pardo\_ADD\_TO\_FRONTIER}{$v_\text{new},i,\text{DB},\text{IA},p$}

\State $\text{IA}\lbrack 0,i+1 \rbrack\gets \text{IA}\lbrack 0,i \rbrack+d;\qquad \text{IA}\lbrack 1,i \rbrack\gets \texttt{True};$

\State $\text{idx}\gets \text{IA}\lbrack 0,i \rbrack$
\State $d\gets\text{deg}(v_\text{new})$
\ParFor{$j=0$ to $p-1$}
    \For{$k=\text{idx}+j\cdot \frac{d}{p}$ to $\text{idx}+(j+1)\cdot \frac{d}{p}-1$}
        \State $\text{DB}\lbrack 0,k \rbrack\gets n$
        \State $\text{DB}\lbrack 1,k \rbrack\gets (k\neq \text{idx})?(k-\text{idx}):-d$
        \State $\text{DB}\lbrack 2,k \rbrack\gets i$
    \EndFor
\EndParFor

\EndFunction
\end{algorithmic}
\end{algorithm}

Algorithm \ref{algo: frontier sampling par} shows the details of Dashboard-based parallel frontier sampling, where all arrays are zero-based. Considering the main loop (lines 16 to 25), we analyze the complexity of the three functions in Algorithm \ref{algo: frontier sampling func}. Denote $\text{COST}_{\text{rand}}$ and $\text{COST}_{\text{mem}}$ as the cost to generate one random number and to perform one memory access, respectively. 

\paragraph*{pardo\_POP\_FRONTIER}

Anytime during sampling, on average, the ratio of valid DB entries (those occupied by current frontier vertices) over total number of DB entries is $1/\eta$. Probability of one probing falling on a valid entry equals $1/\eta$. 
Expected number of rounds for $p$ processors to generate at least 1 valid probing can be shown to be $1/\paren{1-\paren{1-\frac{1}{\alpha}}^p}$, where one round refers to one repetition of lines 5 to 7 of Algorithm \ref{algo: frontier sampling func}. After selection of $v_{\text{pop}}$, $\text{deg}(v_{\text{pop}})$ number of slots needs to be updated to invalid values \verb|INV|. 
Since this operation occurs $(n-m)$ times, the para\_POP\_FRONTIER function incurs $(n-m) \paren{\frac{1}{1-(1-1/\alpha)^p}\cdot \text{COST}_{\text{rand}}+\frac{d}{p}\cdot \text{COST}_{\text{mem}}}$ cost.

\paragraph*{pardo\_CLEANUP}

Each time cleanup of DB happens, we need one traversal of IA to calculate the cumulative sum of indices (slot 1) masked by the status (slot 2), so as to obtain the new location for each valid entries in DB. On expectation, only $\eta\cdot m$ entries of IA is filled, so this step costs $\eta\cdot m$. Afterwards, only the valid entries in DB will be moved to the new, empty DB based on the accumulated shift amount. This translates to $3\cdot m\cdot d$ number of memory operations. The para\_CLEANUP function is fully parallelized. 
The cleanup happens only when DB is full, i.e., $\frac{n-m}{(\eta-1)m}$ times throughout sampling. Thus, the cost of cleanup is $\frac{n-m}{(\eta-1)\cdot m}\cdot \frac{3\cdot m\cdot d}{p}\cdot \text{COST}_{\text{mem}}$. We ignore the cost of computing the cumulative sum as $\eta m\ll 3md$.

\paragraph*{pardo\_ADD\_TO\_FRONTIER}

Adding a new frontier $v_n$ to DB requires appending $\text{deg}(v_n)$ new entries to DB. This corresponds to cost of $(n-m)\cdot \frac{3\cdot d}{p}\cdot \text{COST}_{\text{mem}}$.

Overall cost to sample one subgraph with $p$ processors is:

\begin{align}
\label{eq: frontier single}
\left(\frac{\text{COST}_{\text{rand}}}{1-(1-1/\eta)^p} +\left( 4 + \frac{3}{\eta -1}\right)\frac{d\cdot \text{COST}_{\text{mem}}}{p}\right)\cdot (n-m)
\end{align}

Assuming $\text{COST}_{\text{mem}} = \text{COST}_{\text{rand}}$, we have the following scalability bound:

\begin{theorem}
\label{thm: frontier scale}
For any given $\epsilon > 0$, Algorithm~\ref{algo: frontier sampling} guarantees a speedup of at least $\frac{p}{1+\epsilon}, \forall p\leq \epsilon d\paren{4+\frac{3}{\eta-1}} -\eta$. 
\end{theorem}
\begin{proof}
Note that $\frac{1}{1-(1-1/\eta)^p} \leq \frac{1}{1-\exp(-p/\eta)} \leq  \frac{\eta + p}{p}$. This follows from $\frac{1}{1-e^{-x}} = \frac{1}{1-\frac{1}{e^x}} \leq \frac{1}{1-\frac{1}{1+x}} \leq \frac{x+1}{x}$. Further, since $p \leq \epsilon d(4+3/(\eta-1)) - \eta$, we have $\frac{\eta + p}{p} \leq \frac{\epsilon d(4+3/(\eta-1))}{p}$.
Now, speedup obtained by Algorithm~\ref{algo: frontier sampling} compared to a serial implementation ($p=1$) is
\begin{align*}
    &\frac{\paren{\eta + d(3/(\eta-1) + 4)}(n-m)}{\left(\frac{1}{1-(1-1/\eta)^p} + \frac{d}{p}(3/(\eta-1) + 4)\right)(n-m)}\\
    &\geq \frac{d(3/(\eta-1) + 4)}{\frac{\epsilon d}{p}(3/(\eta-1) + 4) + \frac{d}{p}(3/(\eta-1) + 4)} \geq \frac{p}{1+\epsilon}.
\end{align*}
\end{proof}

Setting $\epsilon=0.5$ and $\eta=3$, Theorem \ref{thm: frontier scale} guarantees good scalability ($p/1.5$) up to $p=2.25\cdot d-3$ processors. 
While the scalability can be high for dense graphs, it is difficult to scale the sampler to massive number of processors on sparse graphs. The total parallelism available is bound by the graph degree. In summary, the parallel Dashboard based frontier sampler (1) leads to lower serial complexity by incremental update on probability distribution, and
    (2) scales well up to $p = \mathcal{O}(d)$.

To further scale the sampler, we exploit task parallelism by a GCN training scheduler. Since the topology of the training graph $\G$ is fixed over the training iterations, sampling and GCN computation can proceed in an interleaved fashion, without any dependency constraints. By Algorithm \ref{algo: gsaint training sample}, the scheduler maintains a pool of sampled subgraphs $\Set{\G_i}$. When $\Set{\G_i}$ is empty, the scheduler launches $p_{\text{inter}}$ frontier samplers in parallel, and fill $\Set{\G_i}$ with subgraphs independently sampled from the training graph $\G$. Each of the $p_{\text{inter}}$ sampler instances runs on $p_{\text{intra}}$ processing units. Thus, the scheduler exploits both intra- and inter-subgraph parallelism. In each training iteration, we remove a subgraph $\Gsub$ from $\Set{\G_i}$, and build a complete GCN upon $\Gsub$. Forward and backward propagation stay unchanged as lines 4 to 13 in Algorithm \ref{algo: gsaint training}.


During construction of $\Set{\G_i}$, total amount of parallelism $p_{\text{intra}}\cdot p_{\text{inter}}$ is fixed on the target platform. The value of $p_\text{intra}$ and $p_\text{inter}$ should be carefully chosen based on the trade-off between the two levels of parallelism. Note that the operations on DB mostly involve a chunk of memory with continuous addresses. This indicates that intra-subgraph parallelism can be well exploited at the instruction level using vector instructions (e.g., AVX). In addition, note that most of the memory traffic going into DB is in a random manner during sampling. This indicates that ideally, DB should be stored in cache. 
As a coarse estimate, with $m=1000$, $\eta=2$, $d=30$, the memory consumption by one DB is $480\text{KB}$ \footnote{We use \texttt{INT32} for slot 1 of DB, and \texttt{INT16} for slots 2 and 3.}. This indicates that DB mostly fits into the private L2 cache size ($256\text{KB}$) in modern shared memory parallel machines. Therefore, during sampling, we bind one sampler to one processor core, and use AVX instructions to parallelize within a single sampler. For example, on a 40-core machine with AVX2, $p_\text{intra}=8$ and $p_\text{inter}=40$. 

\begin{algorithm}
\caption{GCN training with parallel frontier sampler}
\label{algo: gsaint training sample}
\begin{algorithmic}[1]
\renewcommand{\algorithmicrequire}{\textbf{Input:}}
\renewcommand{\algorithmicensure}{\textbf{Output:}}
\Require Training graph $\G(\V,\E,\Hh^{(0)})$; Labels $\Label$; Sampler parameters $m,n,\eta$; Parallelization parameters $p_{\text{inter}}, p_{\text{intra}}$
\Ensure Trained weights $\Set{\Wself^{(\ell)},\Wneigh^{(\ell)}|1\leq\ell\leq L}$
\State $\Set{\G_i}\gets \emptyset$\color{blue}\Comment{Set of unused subgraphs}\color{black}
\While{not terminate}\color{blue}\Comment{Iterate over mini-batches}\color{black}
    \If{$\Set{\G_i}$ is empty}
        \ParFor{$p=0$ to $p_{\text{inter}}-1$}
            \State $\Set{\G_i}\gets \Set{\G_i}\cup \text{SAMPLE}_{\text{G}}(\G,m,n,\eta,p_{\text{intra}})$
        \EndParFor
    \EndIf
    \State $\Gsub\gets $ Subgraph popped out from $\Set{\G_i}$
    \State $\Set{\Vgs^{(\ell)}}, \Set{\E_{\text{GS}}^{(\ell)}}\gets $ GCN construction on $\G_{\text{sub}}$
    \State Forward and backward propagation of GCN
\EndWhile
\State \Return $\Set{\Wself^{(\ell)}},\Set{\Wneigh^{(\ell)}}$
\end{algorithmic} 
\end{algorithm}

\section{Parallel Training Algorithm}
\label{sec: para spmm}

In this section, we present parallelization techniques for the forward and backward propagation. Specifically, 
the graph sampling based GCN enables a simple partitioning scheme that guarantees close-to-optimal feature propagation performance. 

\subsection{Computation Kernels in Training}
\label{sec: para feat prop kernel}
Forward and backward propagation involves two kernels: 
\begin{itemize}
    \item Feature propagation in subgraph (step $\paren{\Ags^{(\ell)}}^T \cdot \Hgs^{(\ell-1)}$);
    \item Dense matrix multiplication (steps involving the weight matrices $\Wneigh^{(\ell)}$ and $\Wself^{(\ell)}$).
\end{itemize}

We study the parallel feature propagation in Section \ref{sec: para feat prop}. As for the computation on dense matrix multiplication, this is a standard BLAS level 2 routine, which can be efficiently parallelized using libraries such as Intel \textsuperscript{\small\textregistered}
 MKL \cite{MKL}. 

\subsection{Parallel Feature Propagation within Subgraph}
\label{sec: para feat prop}

During training, each of the $\Vsub^{(\ell)}$ nodes pulls features from its neighbors, along the inter-layer edges. Essentially, this can be viewed as feature propagation within $\Gsub$. 

A similar problem, label propagation within graphs, has been extensively studied in the literature. State-of-the-art methods based on
vertex-centric \cite{prop_blocking}, edge-centric \cite{xstream} and partition-centric \cite{gpop} paradigms perform vertex partitioning on graphs so that processors can work independently in parallel. The work in \cite{hidden_dim} also performs label partitioning along with graph partitioning when the label size is large. 
In our case, we borrow the above ideas to allow two dimensional partitioning along the graph as well as the feature dimensions. 
However, we also realize that the aforementioned techniques may lead to sub-optimal performance in our GCN, due to two reasons:

\begin{itemize}
    \item The propagated data from each vertex is a long vector rather than a single value label. 
    \item Our graph sizes are small after graph sampling, so partitioning of the graph may not lead to significant advantage. 
\end{itemize}

Below, we analyze the computation and communication costs of feature propagation after graph and feature partitioning. We temporarily ignore the issues of load-balancing and pre-processing overhead, and address these issues later on. 
For conciseness, we use $\G(\V,\E)$ to refer to $\Gsub(\Vsub,\Esub)$.

Suppose we partition $\G$ into $Q_v$ disjoint vertex partitions $\Set{\V^{(i)}|0\leq i\leq Q_v-1}$. Let the set of vertices that send features to vertices in $\V^{(i)}$ be $\V_{\text{src}}^{(i)}=\Set{u|(u,v)\in\E\land v\in \V^{(i)}}$. Note that $\V^{(i)}\subseteq\V_{\text{src}}^{(i)}$, since we follow the design in \cite{graphsage} to add a self-connection to each vertex. 
We further partition the feature vector $\bm{h}_v\in \mathbb{R}^f$ of each vertex $v$ into $Q_f$ equal parts $\Set{\bm{h}_v^{(i)}|0\leq i\leq Q_f-1}$. 
Each of the processors is responsible for propagation of $\Hh^{(i,j)}=\Set{\bm{h}_v^{(j)}|v\in \V_{\text{src}}^{(i)}}$, flowing from $\V_{\text{src}}^{(i)}$ into $\V^{(i)}$ (where $0\leq i\leq Q_v-1$ and $0\leq j\leq Q_f-1$).

Define the ratio $\gamma_{v}=\frac{\size{\V^{(i)}_\text{src}}}{\size{\V}}$. While $\gamma_{v}$ depends on the partitioning algorithm, it is always bound by $\frac{1}{Q_v}\leq \gamma_{v}\leq 1$.

Let $n=\size{\V}$ and $f=\size{\bm{h}_v}$. So $\size{\V^{(i)}}=\frac{n}{Q_v}$ and $\size{\bm{h}_v^{(i)}}=\frac{f}{Q_f}$. 

In our performance model, we assume $p$ processors operating in parallel. Each processor is associated with a private fast memory (cache). The $p$ processors share a slow memory (DRAM). Our objective in partitioning is to minimize the overall processing time in the parallel system. After partitioning, each processor owns $\frac{Q_v\cdot Q_f}{p}$ number of $\Hh^{(i,j)}$, and propagates its $\Hh^{(i,j)}$ into $\V^{(i)}$. Due to the irregularity of graph edge connections, access into $\Hh^{(i,j)}$ is random. On the other hand, using CSR format, the neighbor lists of vertices in $\V^{(i)}$ can be streamed into the processor, without the need to stay in cache. In summary, an optimal partitioning scheme should:
\begin{itemize}
    \item Let each $\Hh^{(i,j)}$ fit into the fast memory;
    \item Utilize all of the available parallelism;
    \item Minimize the total computation workload;
    \item Minimize the total slow-to-fast memory traffic.
\end{itemize}

Each round of feature propagation has $\frac{n}{Q_v}\cdot d\cdot \frac{f}{Q_f}$ computation, and $2\cdot \frac{n}{Q_v}\cdot d+8\cdot n\cdot \gamma_{v}\cdot \frac{f}{Q_f}$ communication (in bytes)\footnote{Given that sampled graphs are small, we use \texttt{INT16} to represent the vertex indices. We use \texttt{DOUBLE} to represent each feature value.}. Computation and computation over $Q_v\cdot Q_f$ rounds are:

\begin{align}
    g_\text{comp}(Q_v,Q_f)&=n\cdot d\cdot f\\
    g_\text{comm}(Q_v,Q_f)&=2\cdot Q_f\cdot n\cdot d+8\cdot Q_v\cdot n\cdot f\cdot \gamma_v
\end{align}

Note that $g_\text{comp}(Q_v,Q_f)$ is independent of the partitioning scheme. We formulate \textit{communication minimization problem}:

\begin{align}\label{eqn:opt_part}
& \underset{Q_v,Q_f}{\text{minimize}}
& & g_\text{comm}(Q_v,Q_f)=2Q_f\cdot nd + 8Q_v\cdot nf\gamma_v \\
& \text{subject to}
& & Q_v Q_f\geq p;\quad \frac{8nf\gamma_v}{Q_f}\leq S_\text{cache};\quad Q_v,Q_f\in \mathbb{Z}^+
\end{align}
Next, we prove that \emph{without any graph partitioning} we can obtain a 2-approximation for this problem for small subgraphs.
\begin{theorem}
$Q_v=1, Q_f=\max\left\{p, \frac{8nf}{S_\text{cache}}\right\}$ results in a 2-approximation of communication minimization problem, irrespective of the partitioning algorithm, for $p\leq \frac{4f}{d}$ and $2nd \leq S_\text{cache}$.
\end{theorem}
\begin{proof}




Note that since $Q_v, Q_f \geq  1$ and $\gamma_v \geq 1/Q_v$, $\forall Q_v, Q_f$:
\begin{align*}
    g_\text{comm}(Q_v, Q_f) \geq 2Q_f nd + 8Q_v nf\frac{1}{Q_v} \geq 8nf.
\end{align*}
Set $Q_v=1$ and $Q_f = \max\left\{p, \frac{8nf}{S_\text{cache}}\right\}$. Clearly, $\gamma_v = 1$.
\paragraph{Case 1, $p \geq \frac{8nf}{S_\text{cache}}$} In this case, $Q_f=p \geq 8nf/S_\text{cache}$. Thus both constraints are satisfied. And,
 \begin{align*}
 &g_\text{comm}\left(1, p\right) =  2ndp + 8nf \\&= 8nf\left(\frac{pd}{4f} + 1 \right)
 \leq  8nf\cdot (1+1) = 16nf\,,
 \end{align*}
due to $p \leq 4f/d$.

\paragraph{Case 2, $p \leq \frac{8nf}{S_\text{cache}}$} In this case, $Q_f = 8nf/S_\text{cache}$ forms a feasible solution. And,
\begin{align*}
     &g_\text{comm}\left(1, \frac{8nf}{S_\text{cache}}\right) = 2nd\frac{8nf}{S_\text{cache}} + 8nf \\
    &= 8nf\left(\frac{2nd}{S_\text{cache}} + 1\right) \leq 8nf\cdot (1+1) = 16nf.
\end{align*}
In both cases, the approximation ratio of our solution is ensured to be: $
\frac{g_\text{comm}\left(1, \max\left\{p, \frac{8nf}{S_\text{cache}}\right\}\right)}{\min g_\text{comm}(Q_v, Q_f)} \leq \frac{16nf}{8nf} = 2.$

Note that this holds for $S_\text{cache}\geq 2nd$, which means for a cache size of 256KB, number of (directed) edges in the subgraph ($nd$) can be up to 128K, which is higher than that of the subgraphs in consideration. Also, since $f \gg d$, $p \leq 4f/d$ holds up to large values of $p$.

\end{proof}





Using typical  values $n\leq 8000$, $f=512$, and $d=15$, then for up to $136$ cores\footnote{Note that $d$ here refers to the average degree of the sampled graph rather than the training graph. Thus, $d$ value here is lower than that in Section \ref{sec: para sample}.}, the total slow-to-fast memory traffic under feature only partitioning is less than 2 times the optimal. 
Recall the two properties (listed at the beginning of this section) that differentiate our case with traditional label propagation. Because the graph size $n$ is small enough, we can find a feasible $Q_f\in\mathbb{Z}^+$ solution to satisfy the cache constraint $\frac{8nf}{Q_f}\leq S_\text{cache}$. Because the value $f$ is large enough, we can find enough number of feature partitions such that $Q_f\geq p$. 
Algorithm \ref{algo: gsaint training feat prop} shows details of our feature propagation.

\begin{algorithm}
\caption{Feature propagation within sampled graph}
\label{algo: gsaint training feat prop}
\begin{algorithmic}[1]
\renewcommand{\algorithmicrequire}{\textbf{Input:}}
\renewcommand{\algorithmicensure}{\textbf{Output:}}
\Require Sampled graph $\G(\V,\E)$; Vertex features $\Set{\bm{h}_v}$; Cache size $S_\text{cache}$; Number of processors $p$
\Ensure Updated features $\Set{\bm{h}_v}$
\State $n\gets \size{\V};\qquad f\gets \size{\bm{h}_v};$
\State $Q_f\gets \max\left\{p,\frac{8nf}{S_\text{cache}}\right\}$
\State Equal partition each $\bm{h}_v$ into $\Set{\bm{h}_v^{(i)}|0\leq i\leq Q_f-1}$
\For{$r=0$ to $Q_f/p-1$}
    \ParFor{$i=0$ to $p-1$}
        \State Propagation of $\Set{\bm{h}_v^{(i+r\cdot p)}|v\in \V}$ into $\V$
    \EndParFor
\EndFor
\State \Return $\Set{\bm{h}_v}$
\end{algorithmic} 
\end{algorithm}

Lastly, the feature only partitioning leads to two more important benefits. Since graph is not partitioned, load-balancing is optimal across processors. Also, the partitioning along features incurs almost zero pre-processing overhead. In summary, the feature propagation in our graph sampling-based GCN achieves 
\begin{enumerate*}
    \item Minimal computation;
    \item Optimal load-balancing;
    \item Zero pre-processing cost;
    \item Low communication volume.
\end{enumerate*}


\section{Experiments}
\label{sec: exp}

\subsection{Experimental Setup}

We conduct our experiments on 4 large scale graph datasets:
\begin{itemize}
    \item PPI: A protein-protein interaction graph \cite{ppi_reddit}. A vertex represents a protein and edges represent interactions. 
    \item Reddit:  A post-post graph \cite{ppi_reddit}. A vertex represents a post. An edge exists between two posts if the same user has commented on both posts.
    \item Yelp: A social network graph \cite{yelp}. A vertex is a user. An edge represents friendship. Vertex attributes are user comments converted from text using Word2Vec \cite{word2vec}.
    \item Amazon: An item-item graph. A vertex is a product sold by Amazon. An edge is present if two items are bought by the same customer. Vertex attributes are converted from bag-of-words of text item descriptions using singular value decomposition (SVD).
    \end{itemize}
The PPI and Reddit datasets are used in \cite{gcn,graphsage}. Note that the Reddit graph is currently the largest one evaluated by state-of-the-art embedding methods, such as \cite{graphsage,fastgcn,kdd_gcn}. We further prepare two graphs of much larger sizes (Yelp, Amazon), to evaluate scaling more thoroughly. Table \ref{tab:dataset} shows the details.



\begin{table}[!ht]
\caption{Dataset Statistics}
    \centering
    \resizebox{\columnwidth}{!}{
    \begin{tabular}{c|ccccc}
        \toprule
        Dataset & Vertices & Edges & Attribute & Classes & Train/Val/Test\\
        \midrule
        PPI & 14,755 & 225,270 & 50 & 121 (M) & 0.66/0.12/0.22\\
        Reddit & 232,965 & 11,606,919 & 602 & 41 (S) & 0.66/0.10/0.24\\
        Yelp & 716,847 & 6,977,410 & 300 & 100 (M) & 0.75/0.15/0.10\\
        Amazon & 1,598,960 & 132,169,734 & 200 & 107 (M) & 0.80/0.05/0.15\\
        \bottomrule
    \end{tabular}}
    \\
    \vspace{.2cm}
    {\scriptsize{}{*} The (M) mark stands for \textbf{M}ulti-class classification, while (S) stands for \textbf{S}ingle-class.}{\scriptsize \par}
    \label{tab:dataset}
\end{table}

For our graph sampling-based GCN, we provide two implementations, in \verb|Python| (with Tensorflow) and \verb|C++|, respectively \footnote{Code available at \texttt{https://github.com/ZimpleX/gcn-ipdps19}}. The two implementations follow an identical training algorithm (Algorithm \ref{algo: gsaint training sample}). We use the \verb|Python| (Tensorflow) version for single threaded accuracy evaluation in Section \ref{sec: exp acc}, since all baseline implementations are provided in \verb|Python| with Tensorflow. We use the \verb|C++| version to measure scalability of our parallel training (Section \ref{sec: exp scale}). The \verb|C++| implementation is necessary, since \verb|Python| and Tensorflow are not flexible enough for parallel computing experiments (e.g., AVX and thread binding are not explicit in \verb|Python|).

We run experiments on a dual 20-Core Intel\textsuperscript{\small\textregistered} Xeon E5-2698 v4 @2.2GHz machine with 512GB of DDR4 RAM. For the \texttt{Python} implementation, we use \texttt{Python} 3.6.5 with Tensorflow 1.10.0. For the \texttt{C++} implementation, the compilation is via Intel\textsuperscript{\small\textregistered} ICC (\verb|-O3| flag). ICC, MKL and OMP are included in Intel\textsuperscript{\small\textregistered} Parallel Studio Xe 2018 update 3. 




\subsection{Evaluation on Accuracy and Efficiency}
\label{sec: exp acc}

Our graph sampling-based GCN model significantly reduces training complexity without accuracy loss. To eliminate the impact of different parallelization strategies on training time, here we run our implementation as well as all the baselines using single thread.  
Figure \ref{fig: exp cap 1}
shows the accuracy (F1 micro score) vs sequential training time plots. Since all baselines report accuracy using 2 layers in their original papers, all measurements here are based on a 2-layer GCN model. Accuracy is tested on the validation set at the end of each epoch. Between the baselines, \cite{graphsage} achieves the highest accuracy and fastest convergence. Compared with \cite{graphsage}, our GCN model achieves higher accuracy on PPI and Reddit, and the same accuracy on Yelp and Amazon. 
To measure training time speedup, we define an accuracy threshold for each of the datasets. Let $a_0$ be the highest accuracy achieved by the baselines. We define the threshold\footnote{Due to stochasticity in training, we allow $0.25\%$ variance in accuracy.}  as $a_0-0.0025$.
Serial training time speedup is calculated with respect to training time for the best performing baseline to reach the threshold divided by training time for our model to reach the threshold. 
We achieve a serial training time speedup of $1.9\times$, $7.8\times$, $4.7\times$ and $2.1\times$ for PPI, Reddit, Yelp and Amazon, respectively.





\subsection{Evaluation on Scalability}
\label{sec: exp scale}

\subsubsection{Scaling of overall training}

\begin{figure*}[htp]
\makebox[0.9\textwidth]{
	\begin{subfigure}[b]{0.248\textwidth}		
    \includegraphics[width=\linewidth]{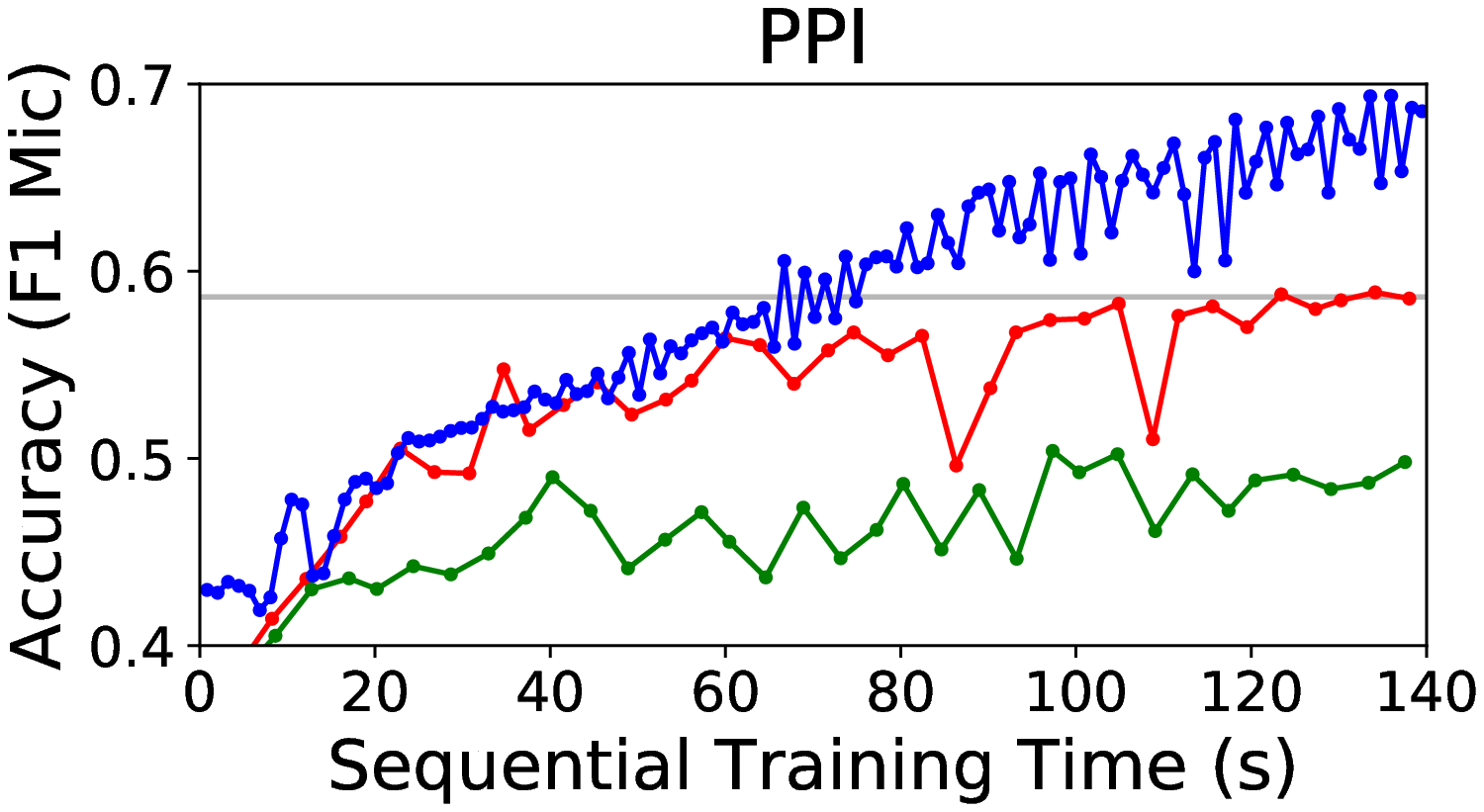}
	\label{fig: exp acc 1 a}
	\end{subfigure}
	\begin{subfigure}[b]{0.235\textwidth}
	\includegraphics[width=1\linewidth]{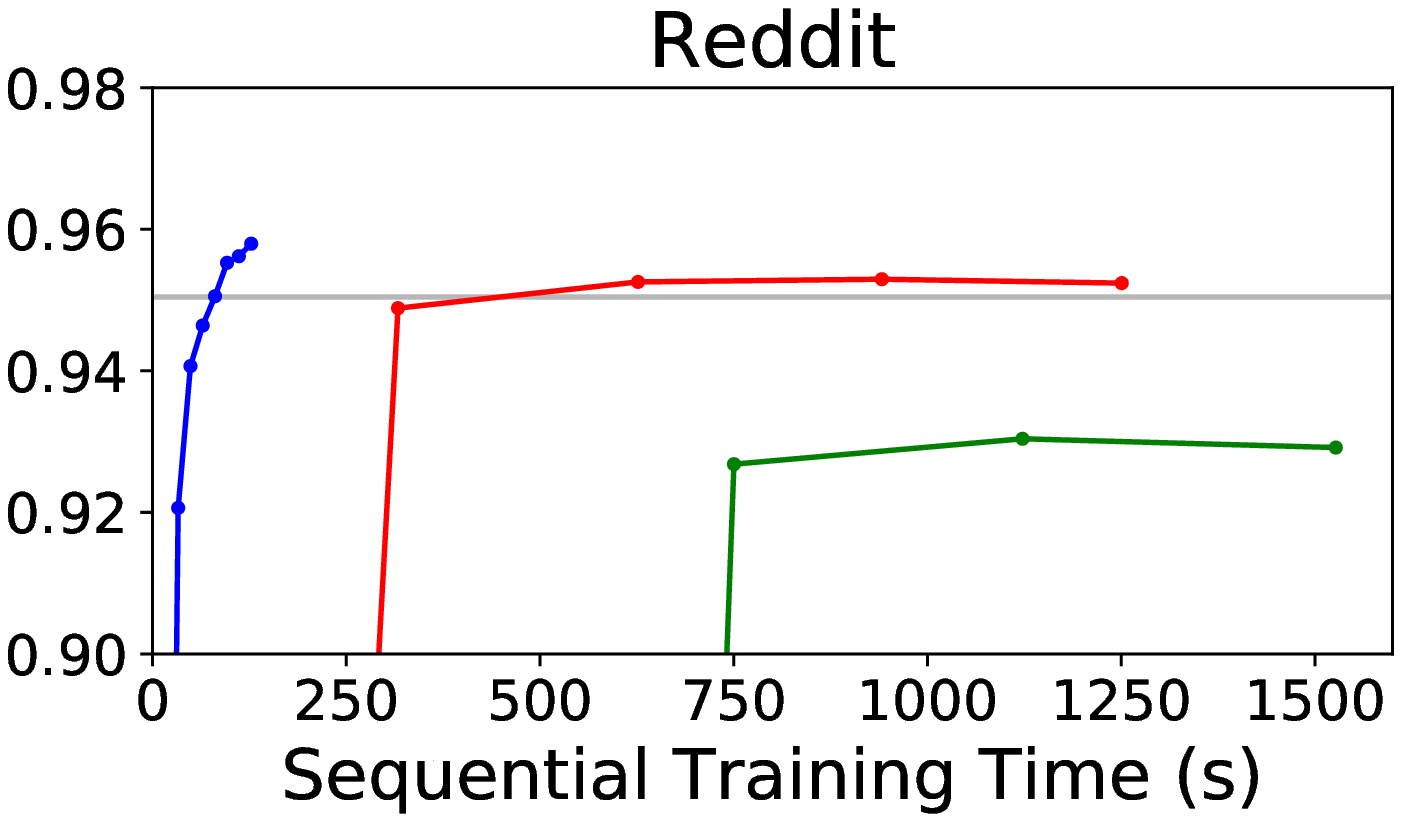}
	\label{fig: exp acc 1 b}
	\end{subfigure}
	\begin{subfigure}[b]{0.235\textwidth}
	\includegraphics[width=1\linewidth]{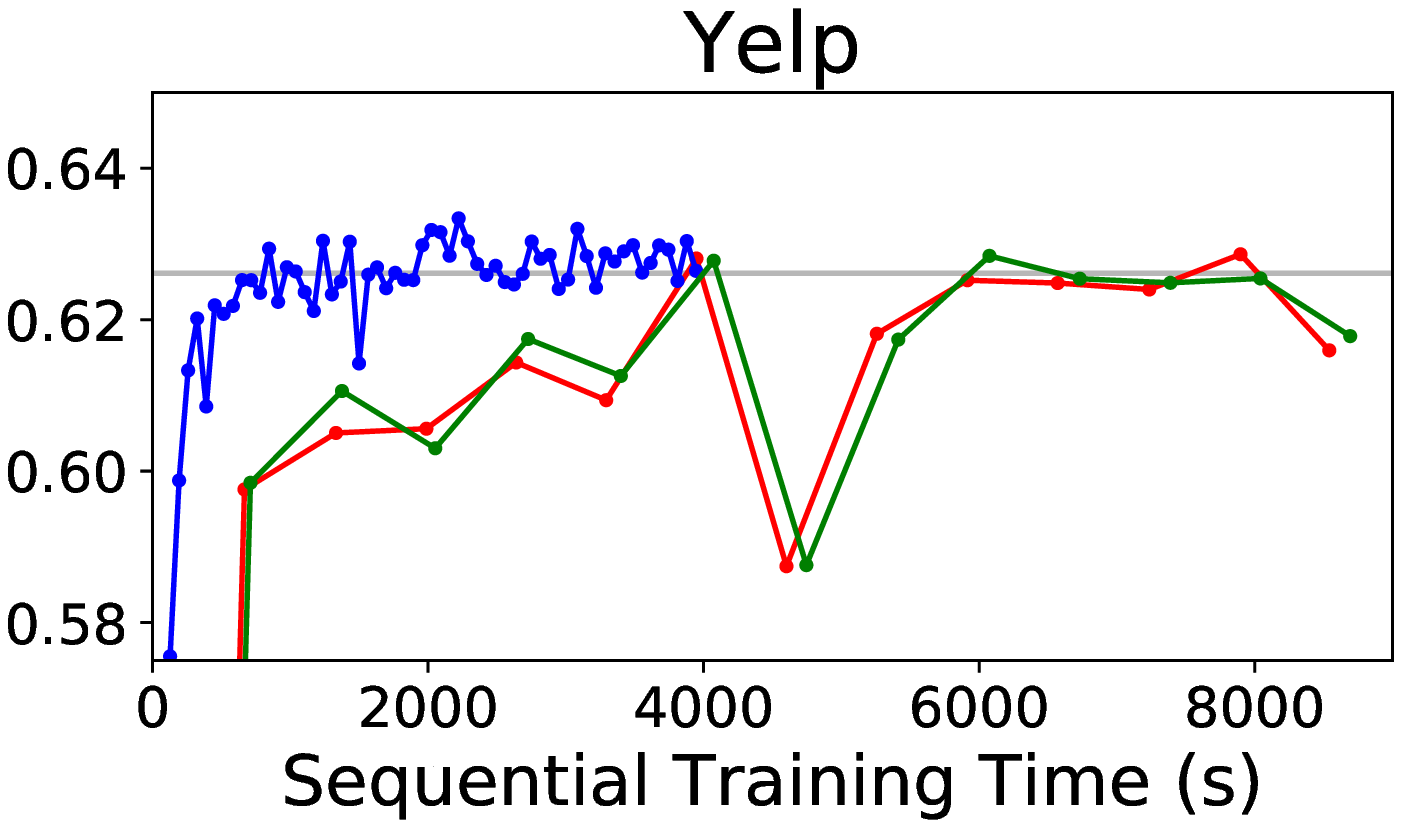}
	\label{fig: exp acc 1 c}
	\end{subfigure}
	\begin{subfigure}[b]{0.235\textwidth}
	\includegraphics[width=1.53\linewidth]{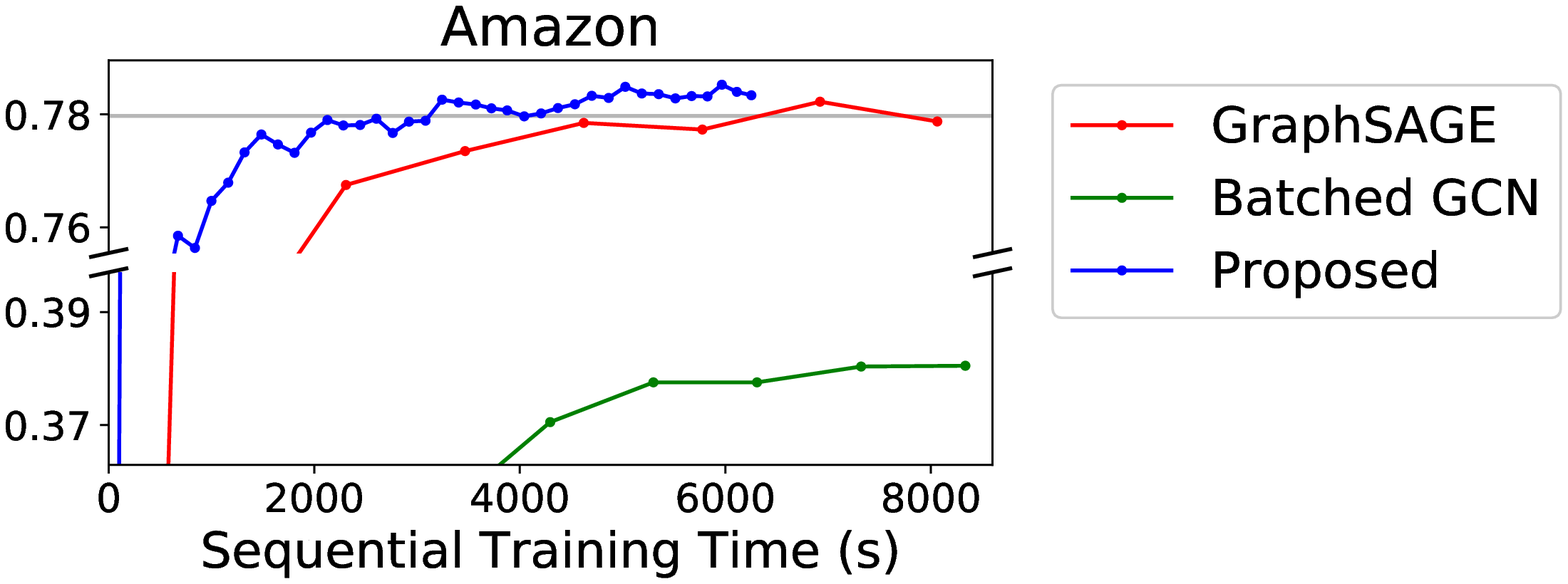}
	\label{fig: exp acc 1 d}
	\end{subfigure}
	}
\vspace{-.4cm}
\caption{Time-Accuracy plot for sequential execution}
\label{fig: exp cap 1}
\vspace{-0.2cm}
\end{figure*}

For our GCN model, Figure \ref{fig: exp scale} shows the parallel training speedup relative to sequential execution. The execution time includes every training steps specified by lines 2 to 8, Algorithm \ref{algo: gsaint training sample} --- sampling (with AVX enabled), feature propagation (forward and backward) and weight application (forward and backward). As before, we evaluate scaling on a 2-layer GCN, with small and large hidden dimensions ($f^{(0)}=f^{(1)}=512$ and $1024$). 
The training is highly scalable, consistently achieving around $20\times$ speedup on 40-cores for all datasets. The performance breakdown in Figure~\ref{fig: exp scale} suggests that
sampling time corresponds to only a small portion of the total time. This is due to
\begin{enumerate*}
    \item low serial complexity of our Dashboard based implementation;
    \item highly scalable design using intra- and inter-subgraph parallelism.
\end{enumerate*}
The main scaling bottleneck is the weight application step using dense matrix multiplication. 
To some extent, the scaling is data dependent. 
High-centrality vertices in the training graph are more likely to be sampled, making caching across training iterations a non-negligible factor. The time on dense matrix multiplication becomes more significant with more processors. 

\begin{figure*}[htp]
\centering
	\begin{subfigure}[b]{1.\textwidth}		
    \includegraphics[width=1\textwidth]{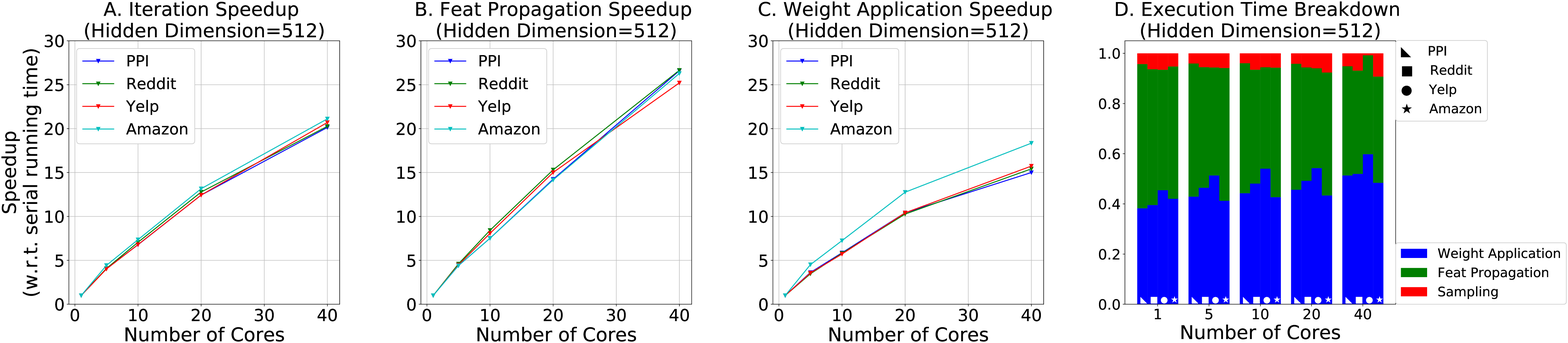}
	\label{fig: exp scale 512}
	\end{subfigure}
	\begin{subfigure}[b]{1.0\textwidth}
	\includegraphics[width=1\linewidth]{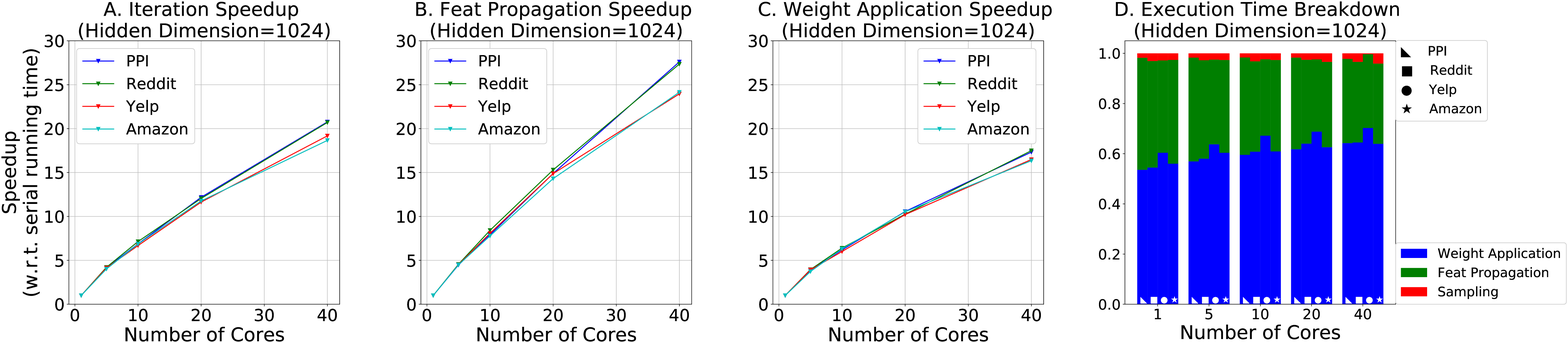}
	\label{fig: exp scale 1024}
	\end{subfigure}
\vspace{-.8cm}
\caption{Scaling evaluation with hidden feature dimensions: $512$ (Upper), $1024$ (Lower)}
\label{fig: exp scale}
\vspace{-0.5cm}
\end{figure*}

%

\subsubsection{Scaling of parallel frontier sampling}

We evaluate the effect of both intra- and inter-subgraph parallelism. The AVX2 instructions supported by our target platform translate to maximum of 8 intra-subgraph parallelism ($p_\text{intra}=8$). The total of 40 Xeon cores makes $1\leq p_\text{inter}\leq 40$. Figure \ref{fg:noavx_vs_avx}A shows the effect of $p_\text{inter}$, when $p_\text{intra}=8$ (i.e., we launch $1\leq p_\text{inter}\leq 40$ independent samplers, where AVX is enabled within each sampler). Sampling is highly scalable with inter-subgraph parallelism. 
We observe that scaling degrades from 20 to 40 cores, due to NUMA effects on cross socket communication --- all the $p_\text{inter}$ samplers on the two sockets keep reading a single copy of the training graph adjacency list, when selecting the next frontier vertices. 
Figure \ref{fg:noavx_vs_avx}B shows the effect of $p_\text{intra}$ under various $p_\text{inter}$. The bars show the speedup of using AVX instructions comparing with otherwise. We achieve around $4\times$ speedup on average. The scaling on $p_\text{intra}$ is data dependent. Depending on the degree distribution of the training graph, there may be significant portion of vertices with less than 8 neighbors. Such load-imbalance explains the discrepancy from the theoretical modeling. 
As a side-note, for highly skewed graphs such as Amazon, we allocate no more than 30 DB entries to a vertex (regardless of its actual degree), in order to upper bound the probability of popping it from the frontier. This helps accuracy improvement on graphs with skewed degree distributions, since it prevents the situation where all subgraphs contain mostly the same set of vertices.

\begin{figure}
\centering
\includegraphics[width=\textwidth/2]{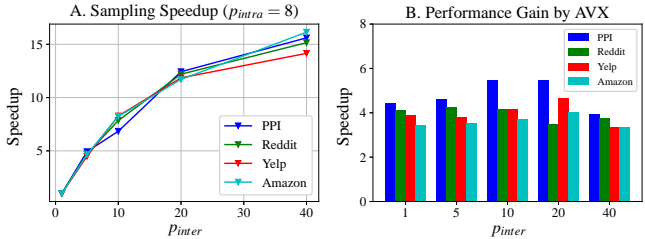}
\caption{Sampling speedup (inter- \& intra-subgraph parallelism)}
\label{fg:noavx_vs_avx}
\end{figure}


\subsubsection{Scaling of feature propagation}

Figure \ref{fig: exp scale} shows the scaling of feature propagation using the partitioning scheme presented in Section \ref{sec: para feat prop}. We achieve high scaling (around $25\times$ speedup on 40 cores) for all datasets for various feature sizes, due to our caching strategy and optimal load-balancing. 



\subsubsection{Scaling of weight application}

As discussed in Section \ref{sec: para feat prop kernel}, the weight application is implemented by \texttt{cblas\_dgemm} routine of Intel\textsuperscript{\small\textregistered} MKL \cite{MKL}. All optimizations are internal to the library. Figure \ref{fig: exp scale} shows the scaling result. 
On 40 cores, average of $16\times$ speedup is achieved. We speculate that the overhead of MKL's internal thread and buffer management is the bottleneck on further scaling. 



\subsection{Deeper Learning}

Though state-of-the-art methods do not provide accuracy results on deeper GCN models, adding more layers in a neural network is proven to be very effective in increasing the expressive power (and thus accuracy) of the network \cite{nn_depth}. 
Here we evaluate the speedup of our GCN implementation compared with \cite{graphsage}, under various number of layers and processors. 
The speedup is calculated by comparing our \texttt{C++} implementation with a \texttt{Python} one, whose parallelism is internally highly optimized via Tensorflow.  
Use Reddit dataset as an example, Table \ref{tab:layer} shows the speedup. Our implementation achieves significantly higher speedup in deeper GCNs ($335\times$ for a 3-layer model, under single thread comparison) as well as with more processors ($37.4\times$  for a standard 2-layer model, using 40 cores). This demonstrates the better scalability of our implementation, and verifies the conclusion in Section \ref{sec: gcn efficiency}. 
Note that due to ``neighbor explosion", for every unit of computation, \cite{graphsage} requires approximately $d_\text{LS}$ times more communication compared with our model. This explains the scalability results with respect to number of processors. In summary,
despite the programming language overheads, the better performance of our design is mainly due to the efficient GCN algorithm and parallelization strategies. Accuracy evaluation for deeper GCN models is out of scope of this paper. 

\begin{table}[!ht]
\caption{Speedup Comparison with Parallelized \cite{graphsage} (Reddit)}
    \centering
    \begin{tabular}{c|ccccc}
        \toprule
          & 1-core & 5-core & 10-core & 20-core & 40-core\\
        \midrule
        1-layer & $2.03\times$ & $4.77\times$ & $9.34\times$ & $17.25\times$ & $23.93\times$ \\
        2-layer & $7.74\times$ & $12.95\times$ & $18.50\times$ & $28.43\times$ & $37.44\times$ \\
        3-layer & $335.36\times$ & $568.93\times$ & $828.25\times$ & $1164.45\times$ & $1306.21\times$\\
        \bottomrule
    \end{tabular}
    \\
    \vspace{.2cm}
    \label{tab:layer}
\end{table}

\section{Conclusion and Future Work}

We presented an accurate, efficient and scalable graph embedding method. Considering the redundant computation in current GCN training, we proposed a graph sampling-based method which ensures accuracy and efficiency by constructing a new GCN on a small sampled subgraph in every iteration. 
We further proposed parallelization techniques to scale the graph sampling and overall training to large number of processors. 

We will extend our graph sampling-based GCN by evaluating impact on accuracy using various sampling algorithms. We will extend the parallel sampler implementation to support a wider class of sampling algorithms. We will also work on the theoretical foundation of the graph sampling-based GCN.
\section*{Acknowledgement}

This material is based on work supported
by the Defense Advanced Research Projects Agency (DARPA) under Contract Number FA8750-17-C-0086, National Science Foundation (NSF) under Contract Numbers CNS-1643351 and ACI-1339756
and Air Force Research Laboratory under Grant Number FA8750-15-1-0185. Any opinions, findings and conclusions or recommendations
expressed in this material are those of the authors and do not necessarily reflect the views of DARPA, NSF or AFRL. 

\bibliographystyle{IEEEtran}

\end{document}